\documentclass[11pt]{article}
\pdfoutput=1

\usepackage[lofdepth,lotdepth,caption=false]{subfig}
\usepackage{fancyhdr}
\usepackage{hyperref}
\usepackage{amsmath, amssymb, graphicx,harvard}
\usepackage{xspace}
\usepackage{braket}
\usepackage{color}
\usepackage{setspace}
\usepackage{bm}
\usepackage{booktabs}
\usepackage{multirow}
\usepackage{siunitx}
\usepackage{subfig}
\usepackage{comment}
\usepackage{titlesec}
\usepackage{sectsty}
\usepackage{secdot} 
\usepackage{amsthm}

\newcommand{\be}{\begin{equation}}
\newcommand{\ee}{\end{equation}}
\newcommand{\bno}{\begin{enumerate}}
\newcommand{\eno}{\end{enumerate}}
\newcommand{\beq}{\begin{eqnarray*}}
\newcommand{\eeq}{\end{eqnarray*}}

\newcommand{\mcal}[1]{\mathcal{#1}}

\newtheorem{theorem}{Theorem}
\newtheorem{lemma}{Lemma}

\newcommand{\CL}{{\cal L}}
\newcommand{\CX}{{\cal X}}
\newcommand{\CH}{{\cal H}}
\newcommand{\bfb}{\mbox{\boldmath $b$}}
\newcommand{\bfd}{\mbox{\boldmath $d$}}
\newcommand{\bfc}{\mbox{\boldmath $c$}}
\newcommand{\bfeps}{\mbox{\boldmath $\epsilon$}}
\newcommand{\bff}{\mbox{\boldmath $f$}}
\newcommand{\bfv}{\mbox{\boldmath $v$}}
\newcommand{\bfs}{\mbox{\boldmath $s$}}
\newcommand{\bfy}{\mbox{\boldmath $y$}}
\newcommand{\bfzero}{\mbox{\boldmath $0$}}

\makeatletter
\newcommand{\mytag}[2]{%
	\text{#1}%
	\@bsphack
	\begingroup
	\@onelevel@sanitize\@currentlabelname
	\edef\@currentlabelname{%
		\expandafter\strip@period\@currentlabelname\relax.\relax\@@@%
	}%
	\protected@write\@auxout{}{%
		\string\newlabel{#2}{%
			{#1}%
			{\thepage}%
			{\@currentlabelname}%
			{\@currentHref}{}%
		}%
	}%
	\endgroup
	\@esphack
}
\makeatother

\usepackage{hyperref}
\usepackage{nameref}

\makeatletter
\let\orgdescriptionlabel\descriptionlabel
\renewcommand*{\descriptionlabel}[1]{%
	\let\orglabel\label
	\let\label\@gobble
	\phantomsection
	\edef\@currentlabel{#1}%
	\let\label\orglabel
	\orgdescriptionlabel{#1}%
}
\makeatother

\begin{document}

\title{\textbf{\Large Low Rank Approximation for Smoothing Spline 
via Eigensystem Truncation}}

\author{Danqing Xu\thanks{Department of Statistics and Applied Probability, University of California, Santa Barbara, CA 93106 (email: \href{mailto:xud@pstat.ucsb.edu}{xud@pstat.ucsb.edu})}
\and 
Yuedong Wang\thanks{Department of Statistics and Applied Probability, University of California, Santa Barbara, CA 93106 (email: \href{mailto:yuedong@pstat.ucsb.edu}{yuedong@pstat.ucsb.edu}) }
}

\date{}

\maketitle	
	
\begin{abstract}
Smoothing splines provide a powerful and flexible means for 
nonparametric estimation and inference. With a cubic time 
complexity, fitting smoothing spline models to large data 
is computationally prohibitive. In this paper, we use the 
theoretical optimal eigenspace to derive a low rank 
approximation of the smoothing spline estimates. We 
develop a method to approximate the eigensystem when it 
is unknown and derive error bounds for the approximate 
estimates. The proposed methods are easy to implement with
existing software. Extensive simulations show that the new 
methods are accurate, fast, and compares favorably against 
existing methods.
\end{abstract}	

\section{Introduction}

As a general class of powerful and flexible modeling 
techniques, spline smoothing has attracted a great deal 
of attention and is widely used in practice. The theory 
of reproducing kernel Hilbert space (RKHS) is used to 
construct various smoothing spline models, thus providing 
a unified framework for theory, estimation, inference, 
and software implementation
\cite{wahba1990spline,gu2013smoothing,wang2011smoothing}.
Many special smoothing spline models such as polynomial, 
periodic, spherical, thin-plate, and L-spline can be 
fitted using the same code \cite{gss,assist}.
The generality and flexibility, however, does come with 
a high computational cost: time and space complexities 
of computing the smoothing spline estimate scale as 
$O(n^3)$ and $O(n^2)$ respectively, where $n$ is the 
sample size. Therefore, fitting smoothing spline models 
with large data is computationally prohibitive. 

Significant research efforts have been devoted to 
reducing the computational burden for fitting smoothing 
spline models. Several low rank approximation methods 
have been proposed in the literature. 
\citeasnoun{pseudospline} approximated the smoother 
matrix by a pseudo-eigendecomposition with orthonormal 
basis functions. \citeasnoun{kim2004smoothing} proposed 
an $O(nq^2)$ method by randomly selecting a subset of 
representers of size $q=o(n)$. Approximating the model 
space using a random subset of representers is not 
efficient since these representers are not selected 
judiciously. \citeasnoun{ma2015efficient} developed an 
adaptive sampling scheme to select subsets of 
representers according to the magnitude of the response
variable. When the roughness and magnitude of the 
underlying function do not coincide, the method in 
\citeasnoun{ma2015efficient} is not spatially adaptive 
\cite{xu2018}. \citeasnoun{wood2003thinplate} used the 
Lanczos algorithm \cite{Lanczos:1950zz} to obtain the 
truncated eigendecomposition for thin-plate splines in 
$O(Kn^2)$ operations with $K$ being the rank of the low 
rank approximation. 

Methods in \citeasnoun{pseudospline}, 
\citeasnoun{kim2004smoothing}, 
\citeasnoun{ma2015efficient} and
\citeasnoun{wood2003thinplate} are low rank approximations.
The optimal approximation strategy 
is to utilize the rapid decaying eigenvalues and obtain 
approximation from eigendecomposition 
\cite{melkman1978,wahba1990spline}. The eigenspaces are 
optimal subspaces (minimal error subspaces) that minimize 
the Kolmogorov n-width \cite{Santin2016}. To the best of 
our knowledge, low rank approximation to the general 
smoothing spline estimate using eigenspace of the 
corresponding RKHS has not been fully studied. Low rank 
approximation for a large matrix has been studied for many 
statistical and machine learning methods including support 
vector machines \cite{fine2002efficient}, kernel principal 
component analysis \cite{zwald2006convergence}, and kernel 
ridge regression (KRR) 
\cite{williams2001,cortes2010impact,Bach13,alaoui2015fast,Yang17}.
For KRR, \citeasnoun{cortes2010impact}, \citeasnoun{Bach13} 
and \citeasnoun{alaoui2015fast} derived error bounds in 
terms of absolute difference, prediction error, and mean 
squared error, respectively. These bounds do not apply to 
smoothing spline directly where the penalty is different 
from that in a KRR. No error bounds have been derived for 
low rank approximations to smoothing spline estimates.

In this paper, we study low rank approximation to general 
smoothing spline estimates using the eigenspace. We will 
approximate the smoothing spline estimates using truncated 
eigensystem and derive error bounds for approximate 
estimates. When the eigensystem is unknown, we will 
approximate functionals applied to eigenfunctions using 
precalculate eigensystem on a set of pre-selected points, 
and derive error bounds for this further approximation. 
We note that error bounds for approximation errors are more 
useful in deciding the trade-off between approximation error 
and computation complexity than asymptotic convergence 
rate in \citeasnoun{kim2004smoothing} and 
\citeasnoun{ma2015efficient}. The proposed method can be 
easily implemented using existing software.

The rest of the paper is organized as follows. Section 
\ref{section:truncation} introduces the low rank 
approximation method and derives error bounds. 
Section \ref{section:pre-selected} presents a method 
for approximating the low rank approximation when 
eigensystem is unknown and derives error bounds for
the additional approximation. Section \ref{section:simulation} 
presents simulation results for the evaluation and 
comparison of the proposed method.

\section{Low Rank Approximation of Smoothing Spline} 
\label{section:truncation}

We review the smoothing spline model in Section 
\ref{section:ssr} and present the low rank approximation
method in Section \ref{subsection:truncation}. Error
bounds are given in Section \ref{section:EB1}.

\subsection{Smoothing spline and its computational cost}
\label{section:ssr}

Consider the general smoothing spline model
\begin{align}
y_i=\CL_i f +\epsilon_i, \quad i=1,\ldots,n, \label{eq:model}
\end{align}
where $f$ belongs to an RKHS $\CH$ on an arbitrary 
domain $\CX$, the unknown function $f$ is observed through 
a known bounded linear functional $\CL_i$, and $\epsilon_i$ 
are iid random errors with mean zero and variance $\sigma^2$. 
For the special case where observations are observed 
directly on the unknown function $f$, $\CL_if = f(x_i)$ and 
$\CL_i$ in this case is called an evaluational functional. 

Let 
$\CH=\CH_0\oplus \CH_1$ where 
$\CH_0=\text{span}\{\phi_1,\ldots,\phi_p\}$ consists 
of functions which are not penalized, and $\CH_1$ is 
an RKHS with reproducing kernel (RK) $R_1$. The smoothing 
spline estimate of the function $f$ is the minimizer of the 
penalized least squares (PLS)
\begin{align}
\sum_{i=1}^{n} \left(y_i-\CL_if\right)^2+n\lambda\Vert P_1 f \Vert^2
\label{PLS}
\end{align}
in $\CH$ where $P_1$ is the projection operator onto the 
subspace $\CH_1$. Let $\bfy=(y_1,\ldots,y_n)^\top$, 
$T=\left.\{ \CL_i\phi_{\nu}\}_{i=1}^n \right._{\nu=1}^p$,
and $\Sigma=\{\CL_i \CL_j R_1 \}_{i,j=1}^n$. 
Assume that $T$ is of full column rank. Then the PLS has a 
unique minimizer 
\cite{wahba1990spline}
\begin{equation}
\hat{f}(x)=\sum_{\nu=1}^{p}d_\nu\phi_\nu(x)+\sum_{i=1}^{n}c_i \xi_i(x),
\label{eq:fhat}
\end{equation}
where $\xi_i(x)=\CL_{i(z)} R_1(x,z)$,
$\CL_{i(z)}$ indicates that $\CL_i$ is applied to what 
follows as a function of $z$, and coefficients 
$\bfc=(c_1,\ldots,c_n)^\top$ and 
$\bfd=(d_1,\ldots,d_p)^\top$ are solutions of 
\begin{equation}  \label{eq:cd}
\begin{aligned}
T \bfd+(\Sigma +n\lambda I)\bfc&=\bfy,\\
T^\top \bfc &= \bfzero.
\end{aligned}
\end{equation}
Solving \eqref{eq:cd} takes $O(n^3)$ floating operations 
\cite{gu2013smoothing}.
Methods in \citeasnoun{kim2004smoothing} and 
\citeasnoun{ma2015efficient} approximate $\CH_1$ 
by the subspace spanned by a subset of representers 
$\{ \xi_1,\ldots,\xi_n \}$ where the subset is either 
selected randomly or adaptively. We will approximate 
$\CH_1$ by its eigenspace which is optimal under 
various circumstances \cite{Santin2016}.

\subsection{Low rank approximation via eigensystem truncation}
\label{subsection:truncation}

Assume that $\CX$ is a compact set in $\mathbb{R}^d$. When 
$R_1$ is continuous and square integrable, then there 
exists an orthonormal sequence of continuous eigenfunctions 
$\Phi_1,\Phi_2,\ldots$ in $L_2(\CX)$ and eigenvalues 
$\delta_1 \geq \delta_2 \geq \ldots \geq 0$ with 
\cite{wahba1990spline} 
\begin{align}
\int_{\CX} R_1(x,z)\Phi_k(z)dz&=\delta_k\Phi_k(x), 
\quad k=1,2,\ldots \label{eq:eigen.int}\\
R_1(x,z)&=\sum_{k=1}^{\infty}\delta_k\Phi_k(x)\Phi_k(z),\\
\int_{\CX}\int_{\CX} R_1(x,z)dxdz
&=\sum_{k=1}^{\infty}\delta_k^2<\infty. \notag
\end{align}
The eigenvalues usually decay fast. For example, the 
Sobolev space 
\begin{equation}
W_2^m[0,1]=\left\{f:~f, f^\prime,\ldots,f^{(m-1)}
\text{ are absolutely continuous, } 
~\int_0^1(f^{(m)})^2dx<\infty\right\}
\label{cubic}
\end{equation}
has eigenvalues $\delta_k\asymp k^{-2m}$ \cite{micchelli1979design}. 

We will leave the space $\CH_0$ unchanged and 
approximate $\CH_1$ by the subspace spanned by the
first $K$ eigenfunctions
$\tilde{\CH}_1=\text{span}\{\Phi_1,\ldots,\Phi_K\}$. 
$\tilde{\CH}_1$ is an RKHS with RK
$\tilde{R}_1(x,z)=\sum_{k=1}^{K}\delta_k\Phi_k(x)\Phi_k(z)$. 
The minimizer of the PLS \eqref{PLS} in the approximate space 
$\CH_K=\CH_0\oplus\tilde{\CH}_1$, 
$\tilde{f}(x)$, provides an approximation to the smoothing 
spline estimate $\hat{f}(x)$. Let 
\begin{align}
\tilde\Sigma=\{\CL_i \CL_j \tilde{R}_1 \}_{i,j=1}^n
= U_1\Delta_1U_1^\top \triangleq ZZ^\top, \label{eq:eigen.de}
\end{align}
where $U_1=\left.\{\CL_i \Phi_{k}\}_{i=1}^n \right._{k=1}^K$ 
is an $n\times K$ matrix, 
$\Delta_1=\text{diag}(\delta_1,\ldots,\delta_K)$, 
$Z=U_1 \Delta_1^{1/2}$, 
and diag($\cdot$) represents a diagonal matrix. 
The approximate estimate 
\begin{equation}
\tilde{f}(x)=\sum_{\nu=1}^p \tilde{d}_\nu\phi_\nu(x)+
\sum_{i=1}^{n}\tilde{c}_i\tilde{\xi}_i(x) ,
\label{eq:approx1}
\end{equation}
where $\tilde{\xi}_i(x)=\CL_{i(z)} \tilde{R}_1(x,z)$, and
coefficients 
$\tilde{\bfc}=(\tilde{c}_1,\ldots,\tilde{c}_n)^\top$ 
and 
$\tilde{\bfd}=(\tilde{d}_1,\ldots,\tilde{d}_p)^\top$ 
are minimizers of 
\begin{align}
\Vert \bfy - T\tilde{\bfd}-\tilde{\Sigma} \tilde{\bfc}  \Vert^2 + 
n\lambda \tilde{\bfc} ^\top \tilde{\Sigma} \tilde{\bfc}. \label{eq:min.target}
\end{align}
Let $\bfb=Z^\top \tilde{\bfc}$, then equation 
\eqref{eq:min.target} reduces to 
\begin{align}
\Vert \bfy - T\tilde\bfd-Z\bfb\Vert^2+n\lambda\Vert\bfb\Vert^2.
\label{eq:min.target.approx}
\end{align}
Equation \eqref{eq:min.target.approx} is the h-likelihood 
of the linear mixed effect (LME) model 
$\bfy=T \tilde \bfd + Z \bfb + \bfeps$ where $\tilde \bfd$
is a vector of fixed effects, $\bfb$ is a vector of random 
effects, and $\bfeps=(\epsilon_1,\ldots,\epsilon_n)^\top$ 
\cite{Wang98}. Therefore existing software for fitting LME
models such as the R package {\tt nlme} may be used to compute 
minimizers $\bfb$ and $\tilde \bfd$.

\subsection{Error Bounds}
\label{section:EB1}

Let $f=f_0+f_1$ where $f_0 \in \CH_0$ and $f_1 \in \CH_1$.
Denote $\hat{f}_0(x)=\sum_{\nu=1}^{p}d_\nu\phi_\nu(x)$ and 
$\hat{f}_1(x)=\sum_{i=1}^{n}c_i \xi_i(x)$
as the estimates of $f_0$ and $f_1$ respectively, and 
$\tilde{f}_0(x)=\sum_{\nu=1}^p \tilde{d}_\nu\phi_\nu(x)$ and
$\tilde{f}_1(x)=\sum_{i=1}^{n}\tilde{c}_i\tilde{\xi}_i(x)$
as the approximations to $\hat{f}_0$ and $\hat{f}_1$ respectively.
Let 
$\hat{\bff}_0=(\hat{f}_0(x_1),\ldots,\hat{f}_0(x_n))^\top$,
$\hat{\bff}_1=(\hat{f}_1(x_1),\ldots,\hat{f}_1(x_n))^\top$,
$\hat{\bff}=(\hat{f}(x_1),\ldots,\hat{f}(x_n))^\top$,
$\tilde{\bff}_0=(\tilde{f}_0(x_1),\ldots,\tilde{f}_0(x_n))^\top$,
$\tilde{\bff}_1=(\tilde{f}_1(x_1),$ \newline
$\ldots,\tilde{f}_1(x_n))^\top$,
and
$\tilde{\bff}=(\tilde{f}(x_1),\ldots,\tilde{f}(x_n))^\top$.
Let $\Vert\cdot\Vert_{2}$, $\Vert \cdot \Vert$, and 
$\Vert \cdot \Vert_F$ denote the $L_2$, Euclidean, and
Frobenius norms respectively. Let
$T=(Q_1 ~ Q_2)(R^\top~ \bfzero)^\top$ be the QR decomposition
where $Q_1$ and $Q_2$ are $n\times p$ and $n \times (n-p)$
matrices, $Q=(Q_1 ~ Q_2)$ is an orthogonal matrix, and $R$
is a $p \times p$ upper triangular and invertible matrix. 

\begin{theorem} \label{bound:truncation}
Assume that $\{\phi_1,\ldots,\phi_p\}$ is a set of 
orthonormal basis for $\CH_0$, and 
$|\CL_i \Phi_{k}| \le \kappa$ for all $i=1,\ldots,n$
and $k=1,2,\ldots$. Then
\begin{align*}
\Vert \tilde{f}_0-\hat{f}_0 \Vert_{2}^2 
&\leq \zeta_2 \Vert\tilde{\Sigma}-\Sigma \Vert_F^2,\\
\Vert \tilde{f}_1-\hat{f}_1 \Vert_{2}^2 
&\leq \zeta_3 \Vert\tilde{\Sigma}-\Sigma \Vert_F^2+
n\kappa^2\Vert \bm c\Vert^2D_K,\\
\Vert \tilde{f}-\hat{f} \Vert_{2}^2
&\leq 2(\zeta_2+\zeta_3)\Vert\tilde{\Sigma}-\Sigma\Vert_F^2+
2n\kappa^2\Vert \bm c\Vert^2D_K, 
\end{align*}
where 
$\zeta_1=\Vert Q_2\Vert_F^6\Vert Q_2^\top\bm y\Vert^2$, 
$\zeta_2=2\lambda_\text{max}(A) (\zeta_1 B \tilde{B}
\Vert\tilde{\Sigma}\Vert_F^2+\Vert {\bm c}\Vert^2)$, 
$\zeta_3=n\kappa^2 C_K\zeta_1B \tilde{B}$, 
$A=T(T^\top T)^{-2}T^\top$, 
$\lambda_\text{max}(A)$ is the largest eigenvalue of $A$,
$B=\sum_{k=1}^{n-p}\lambda_{k,n}^{-2}$,
$\tilde{B}=\sum_{k=1}^{n-p}\tilde\lambda_{k,n}^{-2}$, 
$\tilde\lambda_{k,n}$ and $\lambda_{k,n}$ are eigenvalues of $Q_2^\top\tilde{\Sigma}Q_2$ and $Q_2^\top\Sigma Q_2$ respectively, 
$C_K=\sum_{k=1}^K\delta_k^2$,
and $D_K=\sum_{k=K+1}^\infty\delta_k^2$. 
\end{theorem}

Proof of Theorem \ref{bound:truncation} is given in 
\ref{appendix:theorem1}.

\noindent \textbf{Remarks:} 
\begin{enumerate}
 \item We are interested in the approximation error 
to spline fit with a given dataset. With fixed $\bfy$, orthonormal 
basis of $\CH_0$, eigenfunctions and eigenvalues of 
$\CH_1$, and rank $K$, all terms in the upper bounds can be 
calculated for control of approximation error. 
 \item Since $\delta_k$ is square summable, all terms 
involving $D_K$ can be made arbitrarily small with large 
enough $K$.
 \item Terms involving 
$\Vert\tilde{\Sigma}-\Sigma\Vert_F^2$ can be made 
arbitrarily small with large enough $K$ for common
situations. For example, when 
$\sum_{k=1}^{\infty}\delta_k<\infty$ which is true for
the Sobolev space $W_2^m[0,1]$, since 
$\Vert\tilde{\Sigma}-\Sigma\Vert_F^2
\leq \sum_{i=1}^{n}\sum_{j=1}^{n}
(\sum_{k=K+1}^{\infty}\delta_k\kappa^2)^2
=n^2\kappa^4 (\sum_{k=K+1}^{\infty}\delta_k )^2$,
then $\Vert\tilde{\Sigma}-\Sigma\Vert_F^2$ can be made 
arbitrarily small with large enough $K$. Another example 
is the situation when $\CL_i f=f(x_i)$ and design points 
$x_i$'s are roughly equally spaced, we have 
$\frac{1}{n}\sum_{i=1}^{n}\Phi_k(x_i)\Phi_l(x_i)
\simeq\int\Phi_k(z)\Phi_l(z)dz=\delta_{k,l}$ 
where $\delta_{k,l}$ is the kronecker delta function
\cite{wahba1990spline}. Then 
\begin{eqnarray*}
&&\Vert\tilde{\Sigma}-\Sigma\Vert_F^2 \notag \\
&=&\sum_{i=1}^{n}\sum_{j=1}^{n} 
\sum_{k=K+1}^{\infty}\delta_k^2\Phi_k^2(x_i)\Phi_k^2(x_j)+
2 \sum_{k=K+1}^\infty \sum_{l=k+1}^\infty 
\delta_k \delta_l \sum_{i=1}^{n}\Phi_k(x_i)\Phi_l(x_i)
\sum_{j=1}^{n}\Phi_k(x_j)\Phi_l(x_j) \notag\\
&\simeq& \sum_{k=K+1}^{\infty}\delta_k^2\sum_{i=1}^{n}\Phi_k^2(x_i)\sum_{j=1}^{n}\Phi_k^2(x_j)\notag\\
&\leq& n^2\kappa^4D_K . \notag
\end{eqnarray*}
\end{enumerate}

\section{Low Rank Approximation When the Eigensystem is Unknown} 
\label{section:pre-selected}

\subsection{Approximation to low rank approximation}

When eigenfunctions and eigenvalues are known, we can 
compute $U_1$ and $\Delta_1$ easily without needing to 
perform a matrix eigendecomposition in \eqref{eq:eigen.de}. 
Eigenfunctions and eigenvalues are known for periodic, 
spherical, and trigonometric splines \cite{wahba1990spline}. 
\citeasnoun{amini2012} provide an approximate eigensystem 
for linear spline. Except for special cases, eigenfunctions 
and eigenvalues are in general unknown. We want to avoid 
the direct eigendecomposition of $\Sigma$ since it requires 
$O(n^3)$ computations. The idea behind our approach is to 
approximate eigenfunctions and eigenvalues on a set of 
pre-selected points and save them. We then can approximate 
eigenfunctions at any new $x$ values.

Let $S_N=\{s_1,\ldots,s_N\} \subset  \CX$ be $N$ 
pre-selected points. The discrete version of equation 
\eqref{eq:eigen.int} based on pre-selected $N$ points 
\begin{align}
\frac{1}{N}\sum_{j=1}^{N}R_1(x,s_j)\Phi_k(s_j)\approx\delta_k\Phi_k(x), 
\quad k=1,2,\ldots, \label{eq:discrete}
\end{align}
can be used as an interpolation formula in estimation of 
the eigenfunctions \cite{delves1988computational}. 
Let $\Omega=\{R_1(s_i,s_j)\}_{i,j=1}^N$ and  
$\Omega = V\Gamma V^\top $ be the eigendecomposition where 
$V=(\bfv_1, \ldots, \bfv_N)$ and 
$\Gamma=\text{diag}(\gamma_1,\ldots,\gamma_N)$. 
The approximation in \eqref{eq:discrete} implies that 
$\Omega \Phi_k(\bfs)\approx N \delta_k\Phi_k(\bfs)$ where 
$\bfs=(s_1,\ldots,s_N)^\top$ and
$\Phi_k(\bfs)=\left(\Phi_k(s_1),\ldots,\Phi_k(s_N)\right)^\top$. 
Columns of $V$ and $\gamma_k$'s provide approximations of 
eigenfunctions and eigenvalues: 
$\Phi_k(s_j)\approx\sqrt{N}v_{jk}$ where $v_{jk}$ is the 
$j$th element of $\bfv_k$ and $\delta_k \approx N^{-1} \gamma_k$ \cite{Girolami:2002:OSD:638929.638938}. 
Then $\Phi_k(x) \approx \sqrt{N} \gamma_k^{-1} R_1(x,\bfs) \bfv_k 
\triangleq \check{\Phi}_k(x)$ where $\check{\Phi}_k(x)$ is the
approximate eigenfunction.
For any $\CL_i$, from \eqref{eq:discrete} we have 
$\CL_i \Phi_k \approx \sqrt{N} 
\sum_{j=1}^N \CL_{i(x)} R_1(x,s_j)v_{jk} / \gamma_k
=\sqrt{N} R_{1i}(\bfs)\bfv_k /\gamma_k$ 
where 
$R_{1i}(\bfs)=(\CL_{i(x)} R_1(x,s_1),\ldots, \CL_{i(x)}R_1(x,s_N))$.
Using the first $K \le N$ eigen-vectors and eigen-values only, 
we approximate the RK
\begin{equation}
R_1(x,z) 
\approx \sum_{k=1}^{K}\delta_k \Phi_k(x) \Phi_k(z)
 \approx \sum_{k=1}^{K}\gamma_k^{-1} R_1(x,\bfs)\bfv_k 
 R_1(z,\bfs)\bfv_k\triangleq \check{R}_1(x,z),
\end{equation}
where $R_1(x,\bfs)=(R_1(x,s_1),\ldots,R_1(x,s_N))$.
Then $\Sigma$ is approximated by 
\begin{equation}
\Sigma\approx \{ \CL_{i(x)} \CL_{j(z)} \check{R}_1\}_{i,j=1}^n
\triangleq\check{\Sigma}
=U_2 V_1\Gamma_1^{-1}V_1^\top U_2^\top
=\check{Z}\check{Z}^\top,
\label{checkSigma}
\end{equation}
where $U_2=(R_{11}^\top(\bfs),\ldots,R_{1n}^\top(\bfs))^\top$, 
$V_1=(\bfv_1, \ldots, \bfv_K)$,
$\Gamma_1=\text{diag}(\gamma_1,\ldots,\gamma_K)$, and 
$\check{Z}=U_2V_1\Gamma_1^{-1/2}$. 
The approximate estimate 
\begin{equation}
\check{f}(x)=\sum_{\nu=1}^p \check{d}_\nu\phi_\nu(x)+
\sum_{i=1}^{n}\check{c}_i\check{\xi}_i(x) ,
\label{eq:approx2}
\end{equation}
where $\check{\xi}_i(x)=\CL_{i(z)} \check{R}_1(x,z)$, and
coefficients 
$\check{\bfc}=(\check{c}_1,\ldots,\check{c}_n)^\top$ and 
$\check{\bfd}=(\check d_1,\ldots,\check{d}_p)^\top$ are 
minimizers of \eqref{eq:min.target} with $\tilde{\Sigma}$ 
being replaced by $\check{\Sigma}$.
Again, setting $\check{\bfb}=\check{Z}^\top \check{\bfc}$, 
we solve the minimization problem \eqref{eq:min.target.approx} 
with $Z$ 
being replaced by $\check{Z}$ to obtain the minimizers 
$\check{\bfd}=(\check d_1,\ldots,\check{d}_p)^\top$ and 
$\check{\bfb}=(\check{b}_1,\ldots,\check{b}_K)^\top$. 
Using the fact that 
$\check{\xi}_i(x)=\CL_{i(z)} \check{R}_1(x,z)=
\sum_{k=1}^K \gamma_k^{-1} R_1(x,\bfs) \bfv_k \CL_{i(z)} 
R_1(z,\bfs) \bfv_k$,
the estimate of the function at any point $x$ can be 
calculated as follows:
\begin{align}
\check{f}(x)
&=\sum_{\nu=1}^{p}\check{d}_\nu\phi_\nu(x)+
\sum_{i=1}^{n}\check{c}_i\check{\xi}_i(x)
=\sum_{\nu=1}^{p}\check{d}_\nu\phi_\nu(x)+
\sum_{k=1}^{K}\gamma_k^{-1}R_1(x,\bfs)\bfv_kV_1^\top 
R_1^\top(\bfs)\check{\bfc}\notag \\
&=\sum_{\nu=1}^{p}\check{d}_\nu\phi_\nu(x)+
R_1(x,\bfs)V_1\Gamma_1^{-1/2}\check{Z}^\top\check{\bfc}
=\sum_{\nu=1}^{p}\check{d}_\nu\phi_\nu(x)+
R_1(x,\bfs)V_1\Gamma_1^{-1/2}\check{\bfb}. 
\label{eq:estimate.eigen}
\end{align}

The eigenvectors $V$ and eigenvalues $\gamma_k$'s are 
pre-calculated and stored, thus the proposed approach 
only needs $O(nNK)$ in time to generate the 
approximate truncated eigendecomposition. The 
computation complexity for calculating LME model 
estimate is in the order $O(n(p+K)^2+K^4)$: one time 
matrix calculation (QR decomposition) of order 
$O(n(p+K)^2)$ and Newton-Ralphson iterations of 
order $O(K^4)$ \cite{lindstrom1988}.

\subsection{Error Bounds}

We now derive upper bounds for the approximation 
errors and discuss the impact of rank 
$K$ and the number of pre-selected points $N$. The 
approximation error $\Vert \check{f}-\hat{f}\Vert_{2}^2$ 
is bounded by two approximation errors, 
$\Vert \check{f}-\hat{f}\Vert_{2}^2\leq 
2 \Vert\check{f}-\tilde{f}\Vert_{2}^2+
2\Vert \tilde{f}-\hat{f}\Vert_{2}^2$,
where $\Vert \tilde{f}-\hat{f}\Vert_{2}^2$ represents 
the approximation error due to truncation of the 
eigenfunction sequence and 
$\Vert\check{f}-\tilde{f}\Vert_{2}^2$ represents the 
approximation error due to the approximation of 
the truncated eigenspace. The upper bounds of the 
approximation errors due to truncation are given in
Theorem \ref{bound:truncation}.
The follow theorem
provides upper bounds for the approximation errors 
due to the approximation of the truncated eigenspace.

\begin{theorem} \label{bound:eigen}
Assume that $\{\phi_1,\ldots,\phi_p\}$ is a set of 
orthonormal basis for $\CH_0$, and 
$|\CL_i \Phi_{k}| \le \kappa$ and 
$|\CL_i \check{\Phi}_{k}| \le \kappa'$
for all $i=1,\ldots,n$ and $k=1,2,\ldots$. Then
\begin{align*}
\Vert \check{f}_0-\tilde{f}_0 \Vert_{2}^2 
&\leq\zeta_2'\Vert \check{\Sigma}-\tilde{\Sigma}\Vert_F^2,\\
\Vert \check{f}_1-\tilde{f}_1 \Vert_{2}^2 
&\leq 2\zeta_4 \sum_{k=1}^K
\left \Vert \check{\Phi}_k-\Phi_k \right \Vert_{2}^2
+6\Vert\check{\bm c}\Vert^2\sum_{k=1}^K
\left[\check{\delta}_k^2\sum_{i=1}^n
(\CL_i \check{\Phi}_k-\CL_i \Phi_k)^2\right]\notag\\
&\quad+6n\kappa^2 \Vert \check{\bm c}\Vert^2\sum_{k=1}^K
(\check{\delta}_k-\delta_k)^2+
6\zeta_3'\Vert \check{\Sigma}-\tilde{\Sigma}\Vert_F^2,\\	
\Vert\check{f}-\tilde{f}\Vert_{2}^2
&\leq 4\zeta_4 \sum_{k=1}^K
\left \Vert \check{\Phi}_k-\Phi_k \right \Vert_{2}^2
+12 \Vert\check{\bm c}\Vert^2\sum_{k=1}^K
\left[\check{\delta}_k^2\sum_{i=1}^n(\CL_i \check{\Phi}_k-\CL_i \Phi_k)^2\right]\notag\\
&\quad +12n\kappa^2 \Vert \check{\bm c}\Vert^2\sum_{k=1}^K
(\check{\delta}_k-\delta_k)^2+
(12\zeta_3'+2\zeta_2')\Vert \check{\Sigma}-\tilde{\Sigma}\Vert_F^2,\notag\\ 
\end{align*}
where 
$\zeta_2'=2\lambda_\text{max}(A)(\zeta_1 \tilde{B} \check{B}
\Vert\tilde{\Sigma}\Vert_F^2+\Vert \tilde{\bm c}\Vert^2)$, 
$\zeta_3'=n\kappa^2 C_K\zeta_1 \tilde{B} \check{B}$, 
$\zeta_4=\Vert \check{\bm c}\Vert^2 n \kappa'^2C'_K$,
$\check{B}=\sum_{k=1}^{n-p}\check\lambda_{k,n}^{-2}$, 
$\check \lambda_{k,n}$ are eigenvalues of 
$Q_2^\top\check{\Sigma}Q_2$, and
$C_K'=\sum_{k=1}^K\check\delta_k^2$.
\end{theorem}

Proof of \ref{bound:eigen} is given in \ref{appendix:theorem2}.
The theory of the numerical solution of eigen value problems 
(\citeasnoun{baker1977numerical}, Theorem 3.4 and 3.5) 
shows that if the eigenfunctions $\Phi_k$'s are continuous over 
a compact interval $C_r=[r_1,r_2]$ for $k=1,2,\ldots$, 
$N^{-1}\gamma_k$ and $\check{\Phi}_k$ will converge to 
the true eigenvalue $\delta_k$ and the true eigenfunction $\Phi_k$ 
respectively in the uniform norm:
$\lim_{N\rightarrow\infty}\sup_{\{x\in C_r\}} 
\vert\check{\Phi}_k(x)-\Phi_k(x)\vert=0$,
given $S_N$ is dense enough in the domain. 
Consequently 
$\Vert \check{\Sigma}-\tilde{\Sigma} \Vert_F^2=
\sum_{i=1}^{n}\sum_{k=1}^{n} 
(\sum_{k=1}^{K} (\delta_k \CL_i \Phi_k \CL_j \Phi_k(x_j)-
\check\delta_k\CL_i \check\Phi_k \CL_j\check\Phi_k(x_j))^2$
can be arbitrarily small with large enough $N$. The trade-off between the approximation quality and 
computational time are controlled by both $K$ and $N$.

\section{Simulation Studies} \label{section:simulation}

The cubic spline is one of the most useful smoothing spline 
models. In this section, we explore the performance of our 
low rank approximation method for fitting cubic spline 
models and compare them with existing methods. 

We consider model \eqref{eq:model} with $\CL_if=f(x_i)$ and
three cases of $f$:
$f(x)=\frac{6}{10}\beta_{30,17}(x)+\frac{4}{10}\beta_{3,11}(x)$
(Case 1), 
$f(x)=\frac{1}{3}\beta_{20,5}(x)+\frac{1}{3}\beta_{12,12}(x)+
\frac{1}{3}\beta_{7,30}(x)$ (Case 2), 
and $f(x)=\sin (32\pi x)-8(x-.5)^2$ (Case 3),
where 
$\beta_{p,q}(x)=\frac{\Gamma(p+q)}{\Gamma(p)\Gamma(q)}x^{p-1}(1-x)^{q-1}$. 
Cases 1 and 2 have 2 and 3 bumps respectively. The function 
in Case 3 has periodic oscillations. 
Cases 1, 2, and 3 reflect an increasingly complex ``truth''.
We set $n=10000$, $x_i=i/n$ for 
$i=1,\ldots,n$, and consider two standard deviations of 
random errors: $\sigma=0.1$ and $\sigma=0.2$. 

We fit the cubic spline with model space 
$\CH=W_2^2[0,1]$ and penalty 
$\Vert P_1f\Vert^2=\int_0^1(f^{\prime\prime})^2dx$. 
$W_2^2[0,1]=\CH_0 \oplus \CH_1$ where 
$\CH_0=\text{span}\{ 1, k_1(x) \}$ and $\CH_1$ is an RKHS
with RK $R_1(x,z)=k_2(x)k_2(z)-k_4(|x-z|)$,
$k_r(x)=B_r(x)/r!$, and $B_r$ for $r=0,1,\ldots$
are defined recursively by 
$B_0(x)=1$, $B_r^\prime(x)=r B_{r-1}(x)$, and 
$\int_0^1 B_r(x)dx=0$. 
The fits with the exact RK $R_1$ and a randomly selected 
subset of representers as in \citeasnoun{kim2004smoothing}
are denoted as ALL and RSR respectively.

For our low rank approximation method referred to as EIGEN, we 
compute and save eigensystem of $R_1$ evaluated at grid 
points $S_N=\{s_j=j/N\}_{j=1}^{N}$ with $N=100$. 
Simulation results with $N=1000$ (not shown) are similar. 
We consider five choices of the rank
$K=10,20,30,40,50$. For comparison, 
we apply the Nystr\"{o}m method to derive an
approximation to $\Sigma$. Specifically, let $C$ be an
$n \times K$ matrix formed by $K$ randomly selected 
columns from $n$ columns in $\Sigma$, and $W$ be the 
intersection of the selected rows and columns of $\Sigma$.
Then the Nystr\"{o}m approximation of $\Sigma$ is 
$CW^{-1}C^\top$. Since the running time complexity of 
eigendecomposition on $W$ is $O(K^3)$ and matrix 
multiplication with $C$ takes $O(nK^2)$, the total 
complexity of the Nystr\"{o}m method is $O(K^3+nK^2)$. 
Again, we consider five choices of
the rank $K=10,20,30,40,50$.
We compute ALL and RSR fits using the R functions 
\texttt{ssr} in the \texttt{assist} package \cite{assist}
and \texttt{ssanova} in the \texttt{gss} package 
respectively \cite{gss}.   
For the EIGEN and Nystr\"{o}m methods, the spline 
estimates are 
calculated by the R function \texttt{lme} in the 
\texttt{nlme} package \cite{nlme}.
The smoothing parameter $\lambda$ is selected by the 
GML method \cite{wang2011smoothing}.

For each simulation setting, the experiment is replicated 
for 100 times. 
Table \ref{mse_gml} lists the average MSEs, squared biases, 
and variances for all methods. As expected, the MSEs of 
the EIGEN method are getting closer to those of the exact 
cubic spline estimates as $K$ increases. It indicates that 
the EIGEN method with a large enough $K$ can fully recover 
the exact cubic spline estimate. The EIGEN approach performs 
well and can have smaller MSEs than the ALL and RSR method
with an appropriate choice of $K$. For the EIGEN and Nystr\"{o}m 
methods, bias decreases while variance increases as $K$ 
increases. A good trade-off between bias and variance depends
on the complexity of the true function and standard deviation of the 
random error. For simple functions such as Case 1, the MSE
is dominated by the variance; thus a small $K$ is needed 
to achieve small MSE. For complex functions such as
Case 3, a large $K$ is needed since the MSE is dominated by 
the bias. The EIGEN method has smaller MSEs than the 
Nystr\"{o}m method and needs a smaller $K$ to achieve the same 
level of accuracy. 


\begin{table}
\centering
\caption{average mean squared error (MSE), squared bias 
($\text{Bias}^2$) and variance (Var) (in $10^{-4}$) 
when $n=10000$. ``E'' and ``N'' in abbreviations E10-E50 
and N10-N50 represent EIGEN and Nystr\"{o}m methods 
respectively, while the numbers represent $K$.}
	\label{mse_gml}
	\begin{tabular}{crrrrrrrrr}
		\toprule
		\multirow{2}{*}{Method}&
		\multicolumn{3}{c}{Case 1} &
		\multicolumn{3}{c}{Case 2} &
		\multicolumn{3}{c}{Case 3} \\
		\cmidrule(l{2pt}r{2pt}){2-4}\cmidrule(l{2pt}r{2pt}){5-7}\cmidrule(l{2pt}r{2pt}){8-10}
		& $\text{Bias}^2$ & {Var}& {MSE}& $\text{Bias}^2$ & {Var}& {MSE} & $\text{Bias}^2$ & {Var}& {MSE} \\
		\midrule
		$\sigma=0.1$   &  &  &  & & & &  & &  \\
		ALL & 0.011 & 0.465 & 0.476 & 0.004 & 0.475 & 0.479 & 0.018 & 1.914 & 1.932 \\  
		RSR & 0.022 & 0.365 & 0.387 & 0.005 & 0.387 & 0.391 & 243.116 & 160.124 & 403.240 \\  
		E50 & 0.011 & 0.419 & 0.431 & 0.004 & 0.429 & 0.433 & 0.015 & 0.503 & 0.518 \\ 		 
		E40 & 0.016 & 0.366 & 0.381 & 0.003 & 0.385 & 0.389 & 0.014 & 0.401 & 0.415 \\ 		 
		E30 & 0.044 & 0.295 & 0.339 & 0.003 & 0.314 & 0.316 & 4700.832 & 0.291 & 4701.123 \\ 		 
		E20 & 0.254 & 0.209 & 0.462 & 0.038 & 0.216 & 0.254 & 4796.596 & 0.197 & 4796.793 \\ 		 
		E10 & 13.473 & 0.103 & 13.577 & 56.540 & 0.115 & 56.655 & 4892.651 & 0.107 & 4892.758 \\ 		 
		N50 & 0.038 & 0.426 & 0.463 & 0.005 & 0.430 & 0.435 & 227.180 & 781.722 & 1008.902 \\ 		 
		N40 & 0.057 & 0.520 & 0.577 & 0.011 & 0.646 & 0.658 & 810.058 & 1136.019 & 1946.077 \\ 		 
		N30 & 0.164 & 1.332 & 1.496 & 0.085 & 2.844 & 2.929 & 2044.044 & 1074.665 & 3118.709 \\ 		 
		N20 & 0.611 & 4.880 & 5.491 & 2.663 & 17.775 & 20.438 & 3863.852 & 459.434 & 4323.286 \\ 		 
		N10 & 49.968 & 156.727 & 206.695 & 75.129 & 96.968 & 172.098 & 4825.268 & 44.884 & 4870.152 \\
		\midrule
		$\sigma=0.2$   &  &  &  & & & &  & &  \\
		ALL & 0.042 & 1.500 & 1.542 & 0.023 & 1.419 & 1.441 & 0.058 & 5.782 & 5.840\\
		RSR & 0.047 & 1.398 & 1.445 & 0.022 & 1.348 & 1.370 & 248.838 & 173.654 & 422.492 \\ 
		E50 & 0.041 & 1.457 & 1.498 & 0.022 & 1.385 & 1.406 & 0.031 & 2.106 & 2.137 \\ 
		E40 & 0.042 & 1.385 & 1.428 & 0.021 & 1.323 & 1.344 & 0.022 & 1.707 & 1.729 \\ 		 
		E30 & 0.065 & 1.203 & 1.268 & 0.019 & 1.162 & 1.181 & 4701.181 & 1.242 & 4702.424 \\ 		 
		E20 & 0.269 & 0.887 & 1.156 & 0.049 & 0.844 & 0.893 & 4796.749 & 0.856 & 4797.604 \\ 		 
		E10 & 13.479 & 0.492 & 13.971 & 56.547 & 0.468 & 57.015 & 4892.715 & 0.460 & 4893.175 \\ 		 
		N50 & 0.070 & 1.370 & 1.440 & 0.023 & 1.294 & 1.317 & 224.996 & 785.140 & 1010.136 \\ 		 
		N40 & 0.090 & 1.374 & 1.464 & 0.030 & 1.514 & 1.543 & 812.801 & 1138.041 & 1950.841 \\ 		 
		N30 & 0.179 & 2.089 & 2.267 & 0.094 & 2.599 & 2.693 & 2154.415 & 1052.081 & 3206.496 \\ 		 
		N20 & 0.914 & 8.236 & 9.150 & 1.763 & 15.985 & 17.748 & 3908.169 & 439.851 & 4348.021 \\  
		N10 & 47.832 & 196.206 & 244.039 & 72.846 & 103.378 & 176.223 & 4838.702 & 39.928 & 4878.630 \\
		\bottomrule
	\end{tabular}
\end{table}

All simulations were run on an HP ProLiant DL380 G9 with dual 
Xeon 10 core processors and 128GB of RAM.
Table \ref{time} lists CPU times in seconds per replication 
for all methods. We set $n=10000$ in the above simulation 
such that comparisons can be made with the exact cubic spline 
estimates. To compare computational costs at different sample 
sizes, Figure \ref{fig:time_elapsed} shows CPU times with
$n=500,1000,2500,5000,10000$ (left) and 
$n$ ranged from 20000 to 100000 incremented by 10000 (right). 
It shows that the computational advantage of the EIGEN and 
Nystr\"{o}m methods over existing methods is even more profound 
with larger sample sizes. Figure \ref{fig:time_eigen}
shows that the CPU times of the EIGEN methods increase with 
$n$ linearly.

\begin{table}
	\centering
	\caption{system time elapsed in seconds when $n=10000$.}
	\label{time}
	\begin{tabular}{crrrrrr}
		\toprule
		\multirow{2}{*}{Method}&
		\multicolumn{3}{c}{$\sigma=0.1$} &
		\multicolumn{3}{c}{$\sigma=0.2$} \\
		\cmidrule(l{2pt}r{2pt}){2-4}\cmidrule(l{2pt}r{2pt}){5-7}
		& Case 1 & Case 2& Case 3& Case 1& Case 2& Case 3 \\
		\midrule
		ALL & 206.823 & 216.051 & 233.349 & 199.757 & 208.752 & 233.950 \\ 		 
		RSR & 3.303 & 3.346 & 3.635 & 3.111 & 3.170 & 3.625 \\ 		 
		E50 & 0.959 & 0.967 & 1.025 & 0.963 & 0.951 & 1.037 \\ 		 
		E40 & 0.685 & 0.693 & 0.750 & 0.688 & 0.690 & 0.758 \\ 		 
		E30 & 0.455 & 0.463 & 0.518 & 0.463 & 0.473 & 0.526 \\ 		 
		E20 & 0.279 & 0.282 & 0.305 & 0.284 & 0.283 & 0.310 \\ 		 
		E10 & 0.147 & 0.144 & 0.155 & 0.148 & 0.144 & 0.153 \\ 		 
		N50 & 0.955 & 0.949 & 0.989 & 0.952 & 0.952 & 0.969 \\ 		 
		N40 & 0.689 & 0.683 & 0.688 & 0.679 & 0.679 & 0.684 \\ 		 
		N30 & 0.453 & 0.442 & 0.451 & 0.451 & 0.450 & 0.452 \\ 		 
		N20 & 0.269 & 0.265 & 0.269 & 0.268 & 0.267 & 0.269 \\ 		 
		N10 & 0.132 & 0.128 & 0.132 & 0.132 & 0.132 & 0.131 \\
		\bottomrule
	\end{tabular}
\end{table}

\begin{figure}[htb!]
\includegraphics[width=\textwidth]{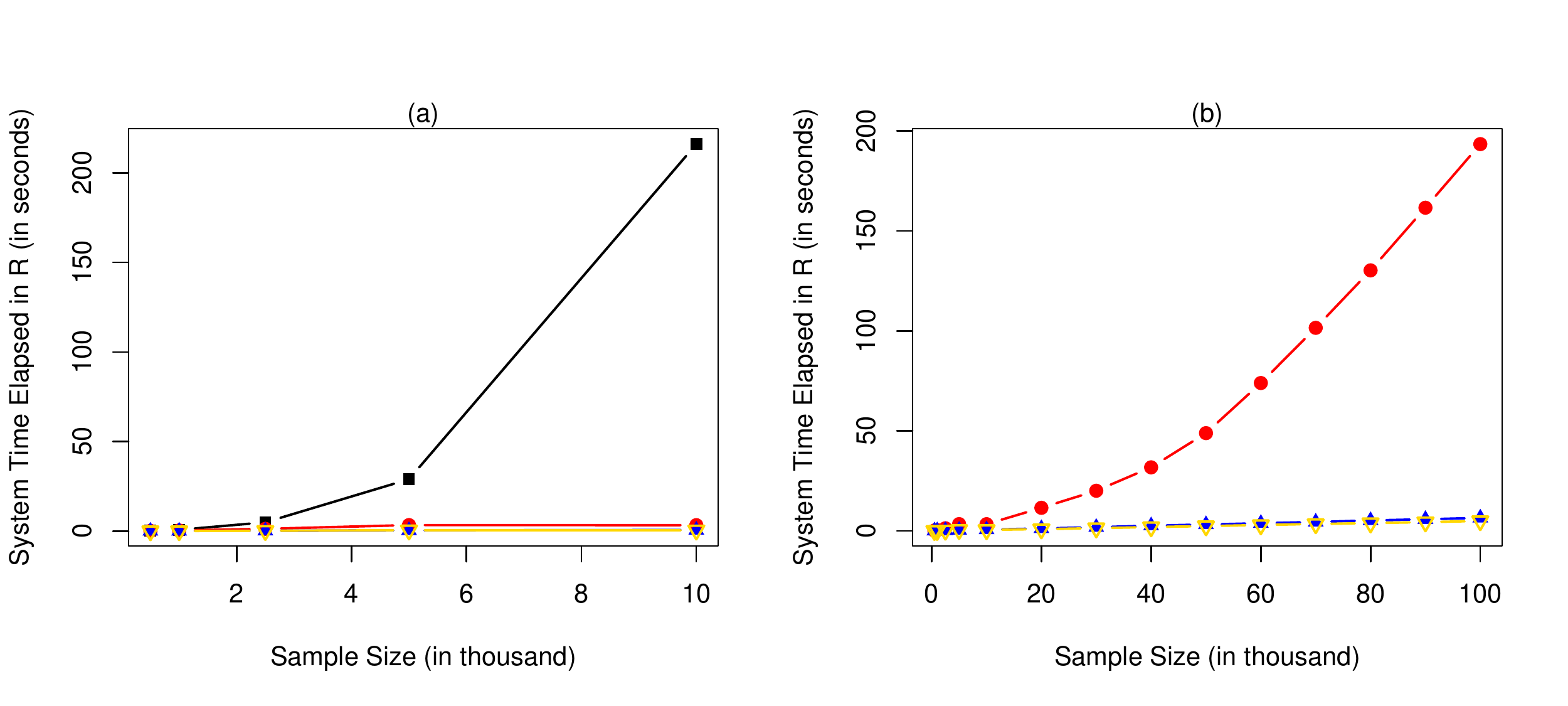}
\caption{System time elapsed in seconds for 
fitting Case 3 with $\sigma=0.1$: (a) ALL (black filled square), 
RSR (red solid circle), EIGEN with $K=30$ (blue solid triangle),
and Nystr\"{o}m with $K=30$ (gold triangle point down) with 
sample sizes $n=500,1000,2500,5000,10000$; 
(b) RSR (red solid circle), EIGEN with $K=30$ (blue solid 
triangle) and Nystr\"{o}m with $K=30$ (gold triangle point down) 
with sample sizes from 20000 to 100000 incremented by 10000.}
\label{fig:time_elapsed}
\centering
\end{figure}

\begin{figure}[htb!]
\centering
\includegraphics[width=0.7\textwidth]{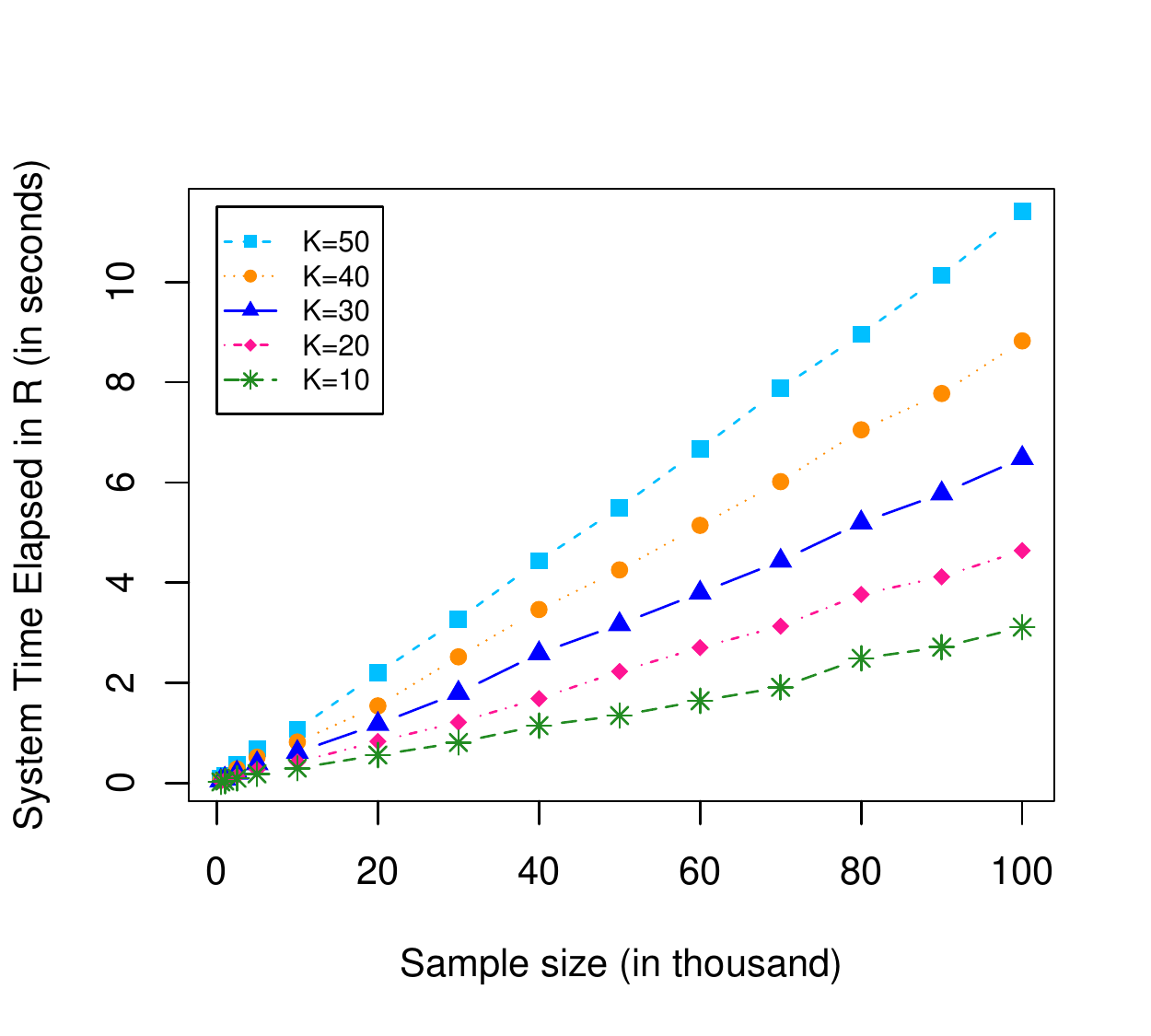}
\caption{System time elapsed in seconds for fitting Case 3 
with $\sigma=0.1$ by the EIGEN approach with $5$ different 
truncation parameters $K$.}
\label{fig:time_eigen}
\end{figure}

\newpage
\noindent {\large\bf Acknowledgements}

This research was supported by a grant from the National 
Science Foundation (DMS-1507620). The authors acknowledge 
support from the Center for Scientific Computing from the 
CNSI, MRL: an NSF MRSEC (DMR1121053).

\bibliographystyle{dcu}
\citationstyle{dcu}  
\bibliography{lrspline}

\newcommand{\Appendix}{\appendix\def\thesection{Appendix~\Alph{section}}
\def\thesubsection{\Alph{section}.\arabic{subsection}}
}

\begin{appendix}
\Appendix    

\renewcommand{\theequation}{A.\arabic{equation}}
\renewcommand{\thesubsection}{A.\arabic{subsection}}
\setcounter{equation}{0}

\section{Proof of Theorem \ref{bound:truncation}}
\label{appendix:theorem1}

\begin{lemma}\label{bound:coef.tilde}
The approximation errors of $\tilde{\bm c}$ and 
$\tilde{\bm d}$ in terms of Euclidean norm are upper bounded as
\begin{align}
\Vert \tilde{\bm c} - \bm c\Vert^2 
&\leq \zeta_1 B \tilde{B} \Vert\tilde{\Sigma}-\Sigma \Vert_F^2,\\
\Vert \tilde{\bm d} - \bm d\Vert^2 
&\leq \zeta_2 \Vert \tilde{\Sigma}-\Sigma\Vert_F^2.
\end{align}
\end{lemma}

\begin{proof}
The solutions to \eqref{eq:cd} are
	\begin{equation}
	\begin{aligned}
	\bm c&=Q_2(Q_2^\top  (\Sigma+n\lambda I) Q_2)^{-1}Q_2^\top \bm y,\\
	\bm d&=R^{-1}Q_1^\top (\bm y-(\Sigma+n\lambda I)\bm c).
	\end{aligned} \label{eq:cdset}
	\end{equation}
The coefficients $\tilde{\bm c}$ and $\tilde{\bm d}$ 
have similar form as \eqref{eq:cdset} with $\Sigma$ 
being replaced by $\tilde{\Sigma}$. 
Let $G=Q_2^\top (\Sigma+n\lambda I)Q_2$,
$\tilde{G}=Q_2^\top (\tilde{\Sigma}+n\lambda I)Q_2$, and
$FDF^\top$ and $\tilde{F}\tilde{D}\tilde{F}^\top$ be
eigendecompositions of $Q_2^\top \Sigma Q_2$ and 
$Q_2^\top \tilde{\Sigma}Q_2$ respectively where $F$ and 
$\tilde{F}$ are $(n-p)\times(n-p)$ orthogonal matrices, and 
$D=\text{diag}(\lambda_{1,n},\ldots,\lambda_{n-p,n})$ and 
$\tilde{D}=\text{diag}(\tilde\lambda_{1,n},\ldots,\tilde \lambda_{n-p,n})$ 
are diagonal matrices with eigenvalues of 
$Q_2^\top \Sigma Q_2$ and $Q_2^\top \tilde{\Sigma}Q_2$. 
Then
\begin{align}
\Vert \tilde{\bm c}-\bm c \Vert^2
&=\Vert Q_2[\tilde{G}^{-1}(\tilde{G}-G)G^{-1}]Q_2^\top \bm y \Vert^2 \notag\\
&=\Vert Q_2\tilde{F}(\tilde{D}+n\lambda I)^{-1}
(Q_2\tilde{F})^\top(\tilde{\Sigma}-\Sigma)Q_2F(D+n\lambda I)^{-1}
(Q_2F)^\top \bm y \Vert^2 \notag \\
&\leq \Vert Q_2\tilde{F}\Vert_F^2 \Vert (\tilde{D}+n\lambda I)^{-1} \Vert_F^2 \Vert (Q_2\tilde{F})^\top \Vert_F^2 \Vert (\tilde{\Sigma}-\Sigma)\Vert_F^2 \Vert Q_2 F \Vert_F^2 \notag\\
&\quad \cdot \Vert (D+n\lambda I)^{-1} \Vert_F^2 \Vert F^\top Q_2^\top \bm y \Vert^2 \notag\\
&= \Vert (\tilde{\Sigma}-\Sigma) \Vert_F^2\left(\sum_{k=1}^{n-p}(\tilde\lambda_{k,n}+n\lambda)^{-2}\right)\left(\sum_{k=1}^{n-p}(\lambda_{k,n}+n\lambda)^{-2}\right)\Vert Q_2\Vert_F^6 \Vert Q_2^\top \bm y \Vert^2\notag\\
&\leq 
\Vert\tilde{\Sigma}-\Sigma \Vert_F^2B \tilde{B}\Vert Q_2\Vert_F^6 \Vert Q_2^\top \bm y \Vert^2\notag\\
&= \zeta_1 B \tilde{B} \Vert\tilde{\Sigma}-\Sigma \Vert_F^2, \label{bd:diff.c}
\end{align}
where we used the facts that the Frobenius norm of a 
vector equals its Euclidean norm, 
$\Vert Q_2F\Vert_F^2=\text{trace}(Q_2FF^\top Q_2^\top)=
\text{trace}(Q_2Q_2^\top)=\Vert Q_2 \Vert_F^2$,
the first inequality holds because of submultiplicativity 
of the Frobenius norm, and the second inequality holds 
because of the triangle inequality and smoothing 
parameter $\lambda$ is non-negative. 

Multiplying the first equation in \eqref{eq:cd} and the 
corresponding first equation for $\tilde{\bm c}$ and 
$\tilde{\bm d}$ by $T^\top $, and then taking the difference, 
we have
$T^\top T(\tilde{\bm d}-\bm d)+T^\top (\tilde{\Sigma}+
n\lambda I)\tilde{\bm c}-T^\top (\Sigma+n\lambda I)\bm c=0$.
Since $T^\top\bm c =\bm 0$ and $T^\top\tilde{\bm c}=\bm 0$
by the second equation in \eqref{eq:cd}, we have
$\tilde{\bm d}-\bm d=
(T^\top T)^{-1}T^\top (\tilde{\Sigma}\tilde{\bm c}-\Sigma\bm c)$ and
\begin{align*}
\Vert \tilde{\bm d}-\bm d \Vert^2
&=(\tilde{\Sigma}\tilde{\bm c}-\Sigma\bm c)^\top A (\tilde{\Sigma}\tilde{\bm c}-\Sigma\bm c)\notag\\
& \leq \lambda_\text{max}(A)\Vert\tilde{\Sigma}\bm c-\Sigma \bm c + \tilde{\Sigma}\tilde{\bm c}-\tilde{\Sigma}\bm c \Vert^2\\
	& \leq 2\lambda_\text{max}(A)\left[\Vert\bm c\Vert^2\Vert\tilde{\Sigma}-\Sigma\Vert_F^2+\Vert\tilde{\Sigma}\Vert_F^2\Vert\tilde{\bm c}- \bm c \Vert^2\right]\\
	& \leq 2\lambda_\text{max}(A)\left[\Vert\bm c\Vert^2\Vert\tilde{\Sigma}-\Sigma\Vert_F^2+\Vert\tilde{\Sigma}\Vert_F^2\zeta_1B \tilde{B} \Vert\tilde{\Sigma}-\Sigma \Vert_F^2 \right]\\
	&= 2\lambda_\text{max}(A)\left(\zeta_1B \tilde{B} \Vert\tilde{\Sigma}\Vert_F^2+\Vert \bm c\Vert^2\right) \Vert \tilde{\Sigma}-\Sigma\Vert_F^2\\
	&= \zeta_2 \Vert \tilde{\Sigma}-\Sigma\Vert_F^2,
\end{align*}
where the second inequality holds by the Cauchy-Schwarz 
inequality, the third inequality holds because of 
submultiplicativity of the Frobenius norm, and the 
fourth inequality hold because of equation \eqref{bd:diff.c}.
\end{proof}

\begin{proof}[Proof of Theorem \ref{bound:truncation}]
 Write the component $\hat{f}_1$ as
\begin{align*}
\hat{f}_1(x)=\sum_{i=1}^{n}c_i \CL_{i(z)} R_1(x,z) =\sum_{k=1}^{\infty}\left[\delta_k 
\sum_{i=1}^{n}c_i \CL_i \Phi_k\right]\Phi_k(x)
\triangleq \sum_{k=1}^{\infty}a_k\Phi_k(x),
\end{align*}
where $a_k=\delta_k \sum_{i=1}^{n}c_i \CL_i \Phi_k$. 
Then the smoothing spline estimate has the form
$\hat{f}(x)=\sum_{\nu=1}^{p}d_\nu\phi_\nu(x)+
\sum_{k=1}^{\infty}a_k\Phi_k(x)$.
Similarly, the low-rank approximation $\tilde{f}$ 
based on $\tilde{\Sigma}$ can be represented as
$\tilde{f}(x)=\sum_{\nu=1}^p \tilde{d}_\nu\phi_\nu(x)
+\sum_{k=1}^{K}\tilde{a}_k\Phi_k(x)$
where 
$\tilde{a}_k=\delta_k \sum_{i=1}^{n}\tilde{c}_i \CL_i \Phi_k$.		
Since 
$\Vert \tilde{f}_0-\hat{f}_0 \Vert^2_{2}=
\int_{\mathcal{X}} (\sum_{\nu=1}^{p}\tilde{d}_\nu\phi_\nu(x)
-\sum_{\nu=1}^{p}d_\nu\phi_\nu(x))^2 dx
=\Vert\tilde{\bm d}-\bm d\Vert^2$,
then we have the upper bound for $\Vert \tilde{f}_0-\hat{f}_0 \Vert^2_{2}$
by Lemma \autoref{bound:coef.tilde}. For the approximation
error $\left  \Vert \tilde{f}_1-\hat{f}_1 \right\Vert^2_{2}$,
we have
\begin{align}
\left  \Vert \tilde{f}_1-\hat{f}_1 \right\Vert^2_{2}
&=\int_{\mathcal{X}}\left[\sum_{k=1}^{K}\tilde{a}_k\Phi_k(x)-\sum_{k=1}^{\infty}a_k\Phi_k(x)\right]^2 dx\notag\\
&=\sum_{k=1}^{K}\left(\tilde{a}_k-a_k\right)^2+\sum_{k=K+1}^{\infty}a_k^2\notag\\
&=\sum_{k=1}^{K}\delta_k^2\left(\sum_{i=1}^n(\tilde{c}_i-c_i)\CL_i \Phi_k\right)^2+\sum_{k=K+1}^{\infty}\delta_k^2\left(\sum_{i=1}^nc_i\CL_i \Phi_k\right)^2\notag\\
&\leq \Vert \tilde{\bm c}-\bm c\Vert^2 
\left(\sum_{i=1}^{n}\sum_{k=1}^{K}\delta_k^2 (\CL_i \Phi_k)^2\right)
+\Vert \bm c \Vert^2\left(\sum_{k=K+1}^{\infty}\sum_{i=1}^{n}
\delta_k^2 (\CL_i \Phi_k)^2 \right) \notag \\
&\leq n\kappa^2 C_K\zeta_1B \tilde{B} \Vert\tilde{\Sigma}-\Sigma \Vert_F^2+n\kappa^2\Vert \bm c\Vert^2D_K\notag\\
&= \zeta_3 \Vert\tilde{\Sigma}-\Sigma \Vert_F^2+n\kappa^2\Vert \bm c\Vert^2D_K.
\end{align}
Finally, using the fact that 
$\Vert \tilde{f}-\hat{f}\Vert_{2}^2
\leq 2\Vert \tilde{f}_0-\hat{f}_0\Vert_{2}+
2\Vert \tilde{f}_1-\hat{f}_1\Vert_{2}$, we have the upper 
bound for the overall function.  
\end{proof}

\section{Proof of Theorem \ref{bound:eigen}}
\label{appendix:theorem2}

Following similar arguments in the proof of Lemma 
\ref{bound:coef.tilde}, it can be shown that 
\begin{align*}
\Vert \check{\bm c}-\tilde{\bm c} \Vert^2 
&\leq \zeta_1 \tilde{B} \check{B}
\Vert\check{\Sigma}-\tilde{\Sigma} \Vert_F^2,\\
\Vert \check{\bm d}-\tilde{\bm d} \Vert^2 
&\leq \zeta_2' \Vert\check{\Sigma}-\tilde{\Sigma}\Vert_F^2.
\end{align*}
Furthermore, 
$\check{f}(x)=\check{f}_0(x)+\check{f}_1(x)=
\sum_{\nu=1}^{p}\check{d}_\nu\phi_\nu(x)+
\sum_{k=1}^K \check{a}_k \check{\Phi}_k(x)$
where 
$\check{a}_k=\check{\delta}_k\sum_{i=1}^{n} \check{c}_i 
\CL_i \check{\Phi}_k$. 
The upper bound for $\Vert \check{f}_0-\tilde{f}_0\Vert_{2}^2$
can be derived similarly as in the proof of Theorem 
\ref{bound:truncation}. We now derive the upper bound for  
$\Vert \check{f}_1-\tilde{f}_1\Vert_{2}^2$.

\begin{align}
\Vert \check{f}_1-\tilde{f}_1\Vert_{2}^2
&=\left \Vert \sum_{k=1}^{K}\check{a}_k\check{\Phi}_k(x)-
\sum_{k=1}^{K}\tilde{a}_k\Phi_k(x) \right \Vert_{2}^2\notag \\
&\leq 2\left\Vert \sum_{k=1}^{K}\check{a}_k\check{\Phi}_k(x)-\sum_{k=1}^{K}\check{a}_k\Phi_k(x)\right \Vert_{2}^2 + 2\left \Vert\sum_{k=1}^{K}\check{a}_k\Phi_k(x)- \sum_{k=1}^{K}\tilde{a}_k\Phi_k(x)\right \Vert_{2}^2  \notag\\
&=2 \left \Vert \sum_{k=1}^{K}\check{a}_k
\left(\check{\Phi}_k(x)-\Phi_k(x)\right)\right \Vert_{2}^2
+ 2\sum_{k=1}^{K}(\check{a}_k-\tilde{a}_k)^2 \notag \notag \\
&= 2(I + II) . \notag 
\end{align}

Moreover, 
\begin{align}
I &=\int_{\mcal{X}}\left(\sum_{k=1}^{K}\check{a}_k\left(\check{\Phi}_k(x)-\Phi_k(x)\right)\right)^2dx \notag\\
& \leq \left(\sum_{k=1}^{K}\check{a}_k^2\right) 
\left(\sum_{k=1}^K\left \Vert \check{\Phi}_k(x)-\Phi_k(x) \right \Vert_{2}^2 \right)\notag\\
&=\sum_{k=1}^{K}\check{\delta}_k^2
\left(\sum_{i=1}^n\check{c}_i \CL_i \check{\Phi}_k \right)^2 
\left(\sum_{k=1}^K\left \Vert \check{\Phi}_k(x)-\Phi_k(x) \right \Vert_{2}^2 \right)\notag\\
&\leq \sum_{k=1}^K\check{\delta}_k^2\left(\sum_{i=1}^n\check{c}_i^2\right)\left(\sum_{i=1}^n \CL_i \check{\Phi}_k^2 \right)\left(\sum_{k=1}^K\left \Vert \check{\Phi}_k(x)-\Phi_k(x) \right \Vert_{2}^2 \right)\notag\\
&\leq \Vert \check{\bm c}\Vert^2 n \kappa'^2\left(\sum_{k=1}^K\check{\delta}_k^2\right) \left(\sum_{k=1}^K\left \Vert \check{\Phi}_k(x)-\Phi_k(x) \right \Vert_{2}^2 \right)\notag \\
&= \zeta_4 \sum_{k=1}^K\left \Vert \check{\Phi}_k(x)-\Phi_k(x) \right \Vert_{2}^2. \label{eq:part3.bd}
\end{align}
The first and second inequalities hold by the 
Cauchy-Schwarz inequality, and the third equality holds 
because of the boundness assumption of $\CL_i \check{\Phi}_k$.

\begin{align*}
II&=\sum_{k=1}^K\left(\check{\delta}_k\sum_{i=1}^n\check{c}_i\CL_i \check{\Phi}_k-\delta_k\sum_{i=1}^n\tilde{c}_i\CL_i \Phi_k\right)^2\\
	\\
&\leq 3 \sum_{k=1}^K \left(\check{\delta}_k\sum_{i=1}^n\check{c}_i
\CL_i \check{\Phi}_k-\check{\delta}_k\sum_{i=1}^n\check{c}_i\CL_i \Phi_k\right)^2
+3\sum_{k=1}^K\left(\check{\delta}_k\sum_{i=1}^n\check{c}_i\CL_i \Phi_k-\delta_k\sum_{i=1}^n\check{c}_i\CL_i \Phi_k\right)^2\\
&\qquad + 3 \sum_{k=1}^K\left(\delta_k\sum_{i=1}^n\check{c}_i\CL_i \Phi_k-\delta_k\sum_{i=1}^n\tilde{c}_i\CL_i \Phi_k\right)^2 \\
& \leq 3\Vert\check{\bm c}\Vert^2\sum_{k=1}^K\left[\check{\delta}_k^2\sum_{i=1}^n(\CL_i \check{\Phi}_k-\CL_i \Phi_k)^2\right] +
3n\kappa^2 \Vert \check{\bm c}\Vert^2\sum_{k=1}^K(\check{\delta}_k-\delta_k)^2
+ 3n\kappa^2C_K\Vert \check{\bm c}-\tilde{\bm c}\Vert^2 ,
\end{align*}
where the first inequality holds by the Cauchy-Schwartz inequality.
Combining $I$, $II$ and upper bound for 
$\Vert \check{\bm c}-\tilde{\bm c}\Vert^2$ we have the upper 
bound for $\Vert \check{f}_1-\tilde{f}_1\Vert_{2}^2$.

Finally, using the fact that 
$\Vert \check{f}-\tilde{f}\Vert_{2}^2
\leq 2\Vert \check{f}_0-\tilde{f}_0\Vert_{2}+
2\Vert \check{f}_1-\tilde{f}_1\Vert_{2}$, we have the upper 
bound for the overall function.  

\end{appendix}
\end{document}